\documentclass[submission,copyright]{eptcs}
\providecommand{\event}{ICLP 2022}
\usepackage{breakurl}             
\usepackage{underscore}           

\usepackage{mathptmx}

\usepackage{amsmath}
\usepackage{amssymb}
\usepackage{amsthm}
\usepackage{graphicx}
\usepackage{float}
\usepackage{verbatim}
\usepackage{enumerate}
\usepackage{mathtools}

\usepackage{xcolor}

\newtheorem{theorem}{Theorem}[section]
\newtheorem{proposition}{Proposition}[section]
\newtheorem{example}{Example}
\newtheorem{lemma}{Lemma}[section]
\newtheorem{corollary}[theorem]{Corollary}
\newtheorem{definition}{Definition}[section]

\newcommand{\Section}[1]{Section~\ref{#1}}

\newcommand{\Theorem}[1]{Theorem~\ref{#1}}

\newcommand{\Lemma}[1]{Lemma~\ref{#1}}

\newcommand{\Corollary}[1]{Corollary~\ref{#1}}

\newcommand{\KB}{\textit{\bf K}\hspace{.01in}}
\newcommand{\Not}{\textit{\bf not}\,}
\newcommand{\OBB}{\sf OB}
\newcommand{\bfK}{{\textit {\bf  K}}}

\newcommand{\boldnot}{{\textit {\bf  not}\,}}

\newcommand{\KBB}{\mathcal{K}}
\newcommand{\boldK}{\bf {K}}

\newcommand{\KBatom}{{\KB-atom}}
\newcommand{\KBatoms}{{{\KB}-atoms}}
\newcommand{\M}{\mathcal{M}}

\newcommand{\MM}{\langle M, M_1 \rangle}
\newcommand{\N}{\mathcal{N}}

\newcommand{\NN}{\langle N, N_1 \rangle}
\newcommand{\MN}{\langle M, N \rangle}
\newcommand{\ff}{{\bf f}}
\newcommand{\uu}{{\bf u}}
\renewcommand{\tt}{{\bf t}}
\newcommand{\mneval}[1]{(I, \MN, \MN)( #1  )}

\newcommand{\mmeval}[1]{\mneval{#1}}
\newcommand{\KBA}{\textsf{KA}}
\newcommand{\OO}{\mathcal{O}}
\newcommand{\OB}[1]{{{\sf OB}_{\OO,#1}}}
\renewcommand{\P}{\mathcal{P}}

\renewcommand{\implies}{\supset}

\newcommand{\head}{head}
\newcommand{\Rule}{rule}

\newcommand{\bodyp}{body^+}
\newcommand{\bodyn}{body^-}

\newcommand{\lfp}[1]{{\bf lfp} ~{#1}}

\newcommand{\citename}[1]{\cite{#1}}
\newcommand{\citeyear}[1]{\cite{#1}}

\newcommand{\union}{\cup}
\newcommand{\intersect}{\cap}
\renewcommand{\iff}{\longleftrightarrow}

\newcommand{\btan}{b\kern0.04em t\kern-0.04eman}

\renewcommand{\email}[1]{\\{\footnotesize #1}}

\title{A Fixpoint Characterization of Three-Valued Disjunctive Hybrid MKNF Knowledge Bases}
\author{Spencer Killen, \qquad\qquad Jia-Huai You
\institute{Department of Computing Science \\University of Alberta\\ Edmonton, Alberta, Canada}
\email{\quad sjkillen@ualberta.ca \quad\qquad jyou@ualberta.ca}
}
\def\titlerunning{A Fixpoint Characterization of Three-Valued Disjunctive Hybrid MKNF Knowledge Bases}
\def\authorrunning{S. Killen and J-H. You}

\begin{document}

\raggedbottom
\maketitle

\begin{abstract}
	The logic of hybrid MKNF (minimal knowledge and negation as failure) is a powerful knowledge representation language that elegantly pairs ASP (answer set programming) with ontologies.
	Disjunctive rules are a desirable extension to normal rule-based reasoning and typically semantic frameworks designed for normal knowledge bases need substantial restructuring to support disjunctive rules.
	Alternatively, one may lift characterizations of normal rules to support disjunctive rules by inducing a collection of normal knowledge bases, each with the same body and a single atom in its head.
	In this work, we refer to a set of such normal knowledge bases as a head-cut of a disjunctive knowledge base.
	The question arises as to whether the semantics of disjunctive hybrid MKNF knowledge bases can be characterized using fixpoint constructions with head-cuts. Earlier, we have shown that head-cuts can be paired with fixpoint operators to capture the two-valued MKNF models of disjunctive hybrid MKNF knowledge bases.
	Three-valued semantics extends two-valued semantics with the ability to express partial information.
	In this work, we present a fixpoint construction that leverages head-cuts using an operator that iteratively captures three-valued models of hybrid MKNF knowledge bases with disjunctive rules.
	This characterization also captures partial stable models of disjunctive logic programs since a program can be expressed as a disjunctive hybrid MKNF knowledge base with an empty ontology.
	We elaborate on a relationship between this characterization and approximators in AFT (approximation fixpoint theory) for normal hybrid MKNF knowledge bases.
\end{abstract}

\section{Introduction}\label{section-introduction}

Lifschitz \citeyear{lifschitznonmonotonic1991} created MKNF, a modal autoepistemic logic, to unify several nonmonotonic logics including answer set programming.
This logic was later extended by Motik and Rosati \citeyear{motikreconciling2010} to form hybrid MKNF knowledge bases.
These knowledge bases couple ontologies with rule-based reasoning to enable highly expressive reasoning.
Knorr et al. \citeyear{knorrlocal2011} extended MKNF (and hybrid MKNF) to three-valued MKNF to enable reasoning with partial information.
Disjunctive hybrid MKNF knowledge bases are more expressive than their normal counterparts, they allow the heads of rules to contain a disjunction.
Three-valued disjunctive hybrid MKNF knowledge bases pose many new, interesting challenges and problems in the field of knowledge representation and reasoning.

For normal hybrid MKNF knowledge bases, Knorr et al. \citeyear{knorrlocal2011} define alternating fixpoint operators for the three-values semantics of normal hybrid MKNF knowledge bases and Liu and You \citeyear{liuyou2021} show that these operators can be recast into the framework of AFT.  This shows a close relationship between Knorr et al.'s fixpoint constructions and AFT for normal hybrid MKNF knowledge bases. More interestingly, Liu and You present a new, richer approximator based on which the well-founded semantics can be computed tractably for a larger class of these knowledge bases. 

Killen and You \citeyear{killenfixpoint2021} characterize the two-valued semantics of disjunctive hybrid MKNF knowledge bases using a collection of induced normal logic programs called head-cuts.
They give an operator that takes a two-valued partition and computes a fixpoint for a single head-cut.
A head-cut is a slice of a disjunctive logic program, a normal program whose models relate to the models of the disjunctive program.
Killen and You also provide a framework for reducing the size of the set of head-cuts needed to characterize models for a disjunctive hybrid MKNF knowledge base.
They show how their operator may be integrated into a solver, but the work is limited to two-valued semantics.
It remains unanswered whether this technique of using head-cuts and fixpoint operators can be applied to the three-valued hybrid MKNF knowledge bases defined by Knorr et al. \citeyear{knorrlocal2011} or partial stable semantics \citename{przymusinskistable1991}.

In \Section{section-motivation}, we give an overview of the motivation for a fixpoint characterization of three-valued models of disjunctive hybrid MKNF.
\Section{section-preliminaries} introduces preliminary definitions that are used throughout this work.
In \Section{section-main}, we present a fixpoint construction that captures the three-valued semantics of disjunctive hybrid MKNF knowledge bases. This operator builds upon and subsumes the operator for two-valued semantics \citename{killenfixpoint2021}.
Next, in \Section{section-aft} we outline the relationship between our characterization and approximators in AFT for normal hybrid MKNF knowledge bases.
Finally, we provide discussion in section \ref{section-discussion}.
% In Appendix A, we show that the three-valued semantics of a disjunctive hybrid MKNF knowledge bases \citename{knorrlocal2011} with empty ontologies coincides with the partial stable model semantics of disjunctive logic programs \citename{przymusinskistable1991}.

\section{Motivation}\label{section-motivation}

MKNF \citename{lifschitznonmonotonic1991} is a framework that permits a variety of extensions.
The framework can express stable model semantics without relying on program transformation. 
One highly desired extension is the ability to reason with both the closed-world and open-world assumptions.
While open-world reasoning, which is employed by description logics and ontologies, requires proof of falsity, closed-world reasoning, which is employed by ASP, provides more intuitionistic reasoning.
Hybrid MKNF is a framework for combining answer set programming with ontologies.
It is faithful to the underlying semantics of the description logic and ASP without increasing the combined reasoning complexity if the ontology's entailment relation can be computed in polynomial time \citename{motikreconciling2010}.
Unlike other hybrid reasoning frameworks, hybrid MKNF constrains external theories to be monotonic which allows for the development of fixpoint operators.
This framework can be leveraged to reframe the semantics of hybrid reasoning with answer set programming in a variety of ways.

The advent of conflict-driven learning has ushered in efficient answer set solvers.
The CDNL algorithm from Gebser et al. \citeyear{gebserconflictdriven2012} offers two methods of capturing stable model semantics.
One method relies on loop formulas and the other on unfounded sets.
To construct loop formulas one needs a dependency graph, however, for hybrid systems that combine standalone ontologies with ASP, this requirement is not so easy to realize.
One would either need to handcraft a method of computing a dependency graph for their ontology or face the tremendous cost of computing a dependency graph in a way that would suit any ontology.
Computing unfounded sets is intractable for normal hybrid MKNF and this carries over to disjunctive hybrid MKNF \citename{killenunfounded}.
While an approximation of unfounded sets can be efficiently computed \citename{killenunfounded}, this approach fails when inconsistencies occur as a result of the ontology.
One promising method of resolving the issues that arise due to these inconsistencies lies in the paraconsistent logic developed by Kaminski et al. \citeyear{kaminskiefficient2015} that builds upon the three-valued extension of hybrid MKNF from Knorr et al. \citeyear{knorrlocal2011}.
The benefits of this logic motivate developing characterizations of the three-valued logic as it is subsumed by the paraconsistent logic.

Fixpoint operators are an attractive alternative to dependency graphs.
They implicitly capture dependencies lazily and warrant interest in their own right.
Approximation fixpoint theory offers a succinct and elegant framework for establishing nonmonotonic semantics \citename{denecker2000approximations}.
However, this framework struggles with disjunctive knowledge bases.
One approach is to lift the framework to support disjunctive knowledge as was done by Heyninck and Arieli \citename{heyninckapproximation2021}.
However, an approach like this creates a new framework and thus requires each piece of work on normal knowledge bases to be heavily altered to support the disjunctive case.
We desire a way of restructuring disjunctive knowledge bases so that prior work on normal knowledge bases can be easily lifted to support disjunctive rules. 
Killen and You define a family of operators for normal knowledge bases induced by a disjunctive knowledge base \citeyear{killenfixpoint2021} and
their framework can be used to promote operators defined on normal knowledge bases automatically to support rules with disjunctive heads.
This approach also allows for the identification of classes of programs and partitions that can be verified in polynomial time for two-valued semantics.

% Partial semantics enable useful reasoning to be performed in scenarios where a three-valued model cannot be extended to a two-valued model.
% For a complex ontology paired with an answer set program, one may wish to make a query that only depends on a small portion of the knowledge base.
% Partial semantics allow for a relaxed model of reasoning when a total model may not exist and can be computed without computing an entire two-valued model.
% For this reason, partial semantics tend to be less expensive than two-valued semantics.
% Our second reason for integrating partial semantics stems from the prospect of constructing an efficient solver for hybrid MKNF.
% A solver rarely deals with two-valued partitions.
% Partial semantics have the property that inconsistent partial partitions cannot be extended to consistent total partitions.
% A solver can leverage this property to prune large segments of the two-valued model search space.
% We hope that the insights into partial semantics uncover new techniques for defining allow conflicts in a conflict-driven two-valued solver that do not rely on a dependency graph.

\section{Preliminaries}\label{section-preliminaries}

MKNF is a modal autoepistemic logic defined by Lifschitz \citeyear{lifschitznonmonotonic1991} which extends first-order logic with two modal operators, $\KB$ and  $\Not$, for minimal knowledge and negation as failure respectively.
The logic was later extended by Motik and Rosati \citeyear{motikreconciling2010} to form hybrid MKNF knowledge bases, which support reasoning with ontologies.
We use Knorr et al.'s \citeyear{knorrlocal2011} definition of three-valued hybrid MKNF knowledge bases.
A (three-valued) \textit{MKNF structure} is a triple $(I, \M, \N)$ where $I$ is a (two-valued first-order) interpretation and $\M = \langle M, M_1 \rangle$ and $\N = \langle N, N_1 \rangle$ are pairs of sets of first-order interpretations and $M \supseteq M_1$, $N \supseteq N_1$.
We use three truth values $\ff$, $\uu$, and $\tt$ with the ordering $\ff < \uu < \tt$. The $min$ and $max$ functions over truth values respect this ordering.
Hybrid MKNF knowledge bases rely on the standard name assumption under which every first-order interpretation in an
MKNF interpretation is required to be a Herbrand interpretation with a countably infinite number of additional constants \citename{motikreconciling2010}.
We use $\Delta$ to denote the set of all these constants.
We use $\phi[\alpha/x]$ to denote the formula obtained by replacing all free occurrences of variable x in $\phi$ with the term $\alpha$.
Using $\phi$ and $\sigma$ to denote MKNF formulas, the evaluation of an MKNF structure is defined as follows:
\begin{align*}
	(I, \M, \N)({p(t_1, \dots ,~t_n)}) & =
	\left\{\begin{array}{ll}
		\tt & \textrm{iff } p(t_1, \dots,~t_n) \textrm{ is true in } I  \\
		\ff & \textrm{iff } p(t_1, \dots,~t_n) \textrm{ is false in } I \\
	\end{array}\right.                                                                                             \\
	(I, \M, \N)({\neg \phi})            & =
	\left\{\begin{array}{ll}
		\tt & \textrm{iff } (I, \M, \N)({\phi}) = \ff \\
		\uu & \textrm{iff } (I, \M, \N)({\phi}) = \uu \\
		\ff & \textrm{iff } (I, \M, \N)({\phi}) = \tt \\
	\end{array}\right.                                                                                             \\
	(I, \M, \N)({\exists x, \phi})      & = max\{(I, \M, \N)({\phi[\alpha/x]}) ~|~\alpha \in \Delta\}                                       \\
	(I, \M, \N)({\forall x, \phi})      & = min\{(I, \M, \N)({\phi[\alpha/x]}) ~|~\alpha \in \Delta\}                                       \\
	(I, \M, \N)({\phi \land \sigma})    & = min((I, \M, \N)({\phi}), (I, \M, \N)({\sigma}))                                                 \\
	(I, \M, \N)({\phi \lor \sigma})     & = max((I, \M, \N)({\phi}), (I, \M, \N)({\sigma}))                                                 \\
	(I, \M, \N)({\phi \subset \sigma})  & = \tt \textrm{ iff } (I, \M, \N)({\phi}) \geq (I, \M, \N)({\sigma}) \textrm{ and $\ff$ otherwise} \\
	(I, \M, \N)({\KB \phi})             & =
	\left\{\begin{array}{ll}
		\tt & \textrm{iff  } (J, \MM, \N)({\phi}) = \tt \textrm{ for all } J \in M    \\
		\ff & \textrm{iff  } (J, \MM, \N)({\phi}) = \ff \textrm{ for some } J \in M_1 \\
		\uu & \textrm{otherwise}
	\end{array}\right.                                                                                             \\
	(I, \M, \N)({\Not \phi})            & =
	\left\{\begin{array}{ll}
		\tt & \textrm{iff  } (J, \M, \NN)({\phi}) = \ff \textrm{ for some } J \in N_1 \\
		\ff & \textrm{iff  } (J, \M, \NN)({\phi}) = \tt \textrm{ for all } J \in N    \\
		\uu & \textrm{otherwise}
	\end{array}\right.
\end{align*}

Intuitively, this logic leverages two sets of interpretations, one for true knowledge and the other for possibly-true knowledge.
A \KBatom{} $\KB a$ is true if $a$ is true in every ``true'' interpretation, $\Not a$ holds if $a$ is false in some ``possibly-true'' interpretation.
$\KB a$ and $\Not a$ are both undefined otherwise.
When we evaluate formulas in this logic, we use a pair of these sets so that $\Not$-atoms may be evaluated independently from \KBatoms{} when checking knowledge minimality.
% We show in Appendix A that these semantics respect the partial stable model semantics of disjunctive logic programs \citename{przymusinskistable1991}.
Note that first-order atoms are evaluated under two-valued interpretations, this is deliberate as, without modal operators, the semantics is essentially the same as first-order logic.
Also note that under three-valued MKNF, logic implication $\phi \subset \sigma$ may not be logically equivalent to $\phi \vee \neg \sigma$ unless both $\phi$ and $\sigma$ are first-order formulas.
% This connective is defined analogously to rules in the partial semantics for disjunctive logic programs \citename{przymusinskistable1991} and we formally show this relationship in Appendix A.

Knorr et al. define their three-valued semantics for the entire language of MKNF \citeyear{knorrlocal2011} which subsumes disjunctive hybrid MKNF knowledge bases.
A disjunctive hybrid MKNF knowledge base contains a program and ontology both are restricted MKNF formulas which we will now define.

An (MKNF) program $\P$ is a set of (MKNF) rules. A rule $r$ is written as follows:
\begin{align*}\label{MKNFRules}
	\KB h_0,\dots,~ \KB h_i \leftarrow \KB p_0,\dots,~\KB p_j,~ \Not n_0,\dots,~\Not n_k
\end{align*}
In the above, $h_0, p_0, n_0, \dots, h_i, p_j, n_k $ are function-free first-order atoms
of the form $p(t_0, \dots,~t_n )$ where $p$ is a predicate and $t_0, \dots,~t_n$ are either constants or variables. Such a rule is called {\em normal} if $i=0$.
An MKNF formula $\phi$ is \textit{ground} if it does not contain variables.
The corresponding MKNF formula for a rule $r$ is as follows:
\begin{align*}
	\pi(r) = \forall \vec{x},~ \KB h_0 \lor \dots \lor  \KB h_i \subset \KB p_0 \land \dots \land \KB p_j \land \Not n_0 \land \dots \land \Not n_k
\end{align*}
where $\vec{x}$ is a vector of all variables appearing in the rule.
We will use the following abbreviations:
\begin{align*}
	\pi(\P) = \bigwedge\limits_{r \in \P} \pi(r)
\end{align*}
\begin{align*}
	\head(r)       & = \{ \KB h_0, \dots,~ \KB h_i \}       &
	\bodyp(r)      & = \{ \KB p_0, \dots,~ \KB p_j \}         \\
	\bodyn(r)      & = \{ \Not n_0, \dots,~ \Not n_k \}     &
	\KB(\bodyn(r)) & = \{ \KB a ~|~ \Not a \in \bodyn(r) \}
\end{align*}

A {\em disjunctive hybrid MKNF knowledge base} (or disjunctive knowledge base for short) $\KBB = (\OO, \P)$ consists of an ontology $\OO$, which is a decidable description logic (DL) knowledge base translatable to first-order logic, and a program $\P$. $\KBB$ is called {\em normal} if all rules in $\P$ are normal.
We use $\pi(\OO)$ to denote the translation of $\OO$ to first-order logic and write $\pi(\KBB)$ to mean $\pi(\P) \land \KB \pi(\OO)$.
A {\em (three-valued) MKNF interpretation (pair)} $(M, N)$ is a pair of sets of first-order interpretations where $\emptyset \subset N \subseteq M$.
We say an MKNF interpretation $(M, N)$ \textit{satisfies} a knowledge base $\KBB = (\OO, \P)$ if for each $I \in M$,
$(I, \MN, \MN) (\pi (\KBB)) = \tt$.

\begin{definition} \label{model}
	A three-valued MKNF interpretation pair $(M, N)$ is a \textit{(three-valued) MKNF model} of a disjunctive hybrid MKNF knowledge base $\KBB$ if $(M, N)$ satisfies $\pi(\KBB)$  and for every three-valued MKNF interpretation pair $(M', N')$ where $M \subseteq M'$, $N \subseteq N'$, and $(M, N) \not= (M', N')$ we have some $I \in M'$ s.t. $(I, \langle M', N' \rangle, \MN)(\pi(\KBB)) \not= \tt$.
\end{definition}

Note that the second condition of our definition differs slightly from the original definition from Knorr et al \citeyear{knorrlocal2011}. They require that $M' = N'$ if $M = N$; we show that this condition is not needed for disjunctive hybrid MKNF knowledge bases while Knorr et al.'s definition applies to all MKNF formulas.
\begin{proposition}
	Let $\KBB$ be a disjunctive hybrid MKNF knowledge base and let $(M, N)$ and $(M', N')$ be MKNF interpretations of $\KBB$ such that $M \subseteq M'$, $N \subseteq N'$, $(M, N) \not= (M', N')$,  $(M, N)$ satisfies $\pi(\KBB)$, and $\forall I \in M', (I, \langle M', N' \rangle, \langle M, N \rangle)(\pi(\KBB)) = \tt$.
	We have ${\forall I \in M', (I, \langle M', M' \rangle, \langle M, N \rangle)(\pi(\KBB)) = \tt}$
\end{proposition}
% \begin{proof}
	% TODO "something missing here"
% 	First we show for $\pi(\P)$.
% 	For each rule $r \in \P$, we have $\forall I \in M', (I, \langle M', N' \rangle, \langle M, N \rangle)(\pi(r)) = \tt$
% 	Let $r$ be a rule from $\P$ such that its body is not false w.r.t. $(M, N)$.
% 	If $(M', N')$ evaluates $\bodyp(r)$ as undefined, then $(M', M')$ evaluates $\bodyp(r)$ as false and then $${\forall I \in M', (I, \langle M', M' \rangle, \langle M, N \rangle)(\pi(r)) = \tt}$$
% 	Now we assume that $(M', N')$ evaluates $\bodyp(r)$ as true.
% 	Because of the initial condition on $(M', N')$, this interpretation must evaluate $\head(r)$ as true and $(M', M')$ will as well, thus 
% 	$${\forall I \in M', (I, \langle M', M' \rangle, \langle M, N \rangle)(\pi(\P)) = \tt}$$
% 	The case for $\KB \pi(\OO)$ comes easily. By the initial condition, we have $\forall I \in M', \pi(\OO) \models I$.
% \end{proof}

If a knowledge base has an MKNF model, we say it is {\em MKNF-consistent}. If it does not have one, then it is {\em MKNF-inconsistent.}
In the rest of this paper, we assume that a given hybrid MKNF knowledge base $\KBB = (\OO, \P)$ is {\em DL-safe},
which ensures the decidability by requiring each variable in a rule $r \in \P$ to appear inside some predicate of $\bodyp(r)$ that does not appear in $\OO$.
Throughout this work, and without loss of generality \citename{knorrlocal2011},
we assume rules in $\P$ are ground.

We use $\KBA(\KBB)$ to denote the following:
$$\KBA(\KBB) = \bigg\{ \KB a ~|~ r \in \P,~ \KB a \in \head(r) \union \bodyp(r) \union \KB(\bodyn(r)) \bigg\}$$
and $\OB{S}$ to denote the objective knowledge of a set $S \subseteq \KBA(\KBB)$:
% w.r.t. $\KBB$:
$$ \OB{S} = \big\{  \pi(\OO) \big\} \union \big\{ a ~|~ \KB a \in S \}$$

Sometimes it is convenient to restrict our focus to the \KBatoms{} in $\KBA(\KBB)$. A {\em (partial) partition} $(T, P)$ of $\KBA(\KBB)$ is a pair where $P \supseteq T$. It partitions the \KBatoms{} in $\KBA(\KBB)$ to be either true ($\KBA(\KBB) \intersect T$), false ($\KBA(\KBB) \setminus P$), or undefined ($\KBA(\KBB) \intersect (P \setminus T)$).

Knorr et al. \citeyear{knorrlocal2011} define that an MKNF interpretation pair $(M, N)$ induces a partition $(T, P)$ if for each $\KB a \in \KBA(\KBB)$:
\begin{itemize}
	\item $\KB a \in T$ if $\forall I \in M, \mmeval{\KB a} = \tt$,
	\item $\KB a \not\in P$ if $\forall I \in M, \mmeval{\KB a} = \ff$, and
	\item $\KB a \in P \setminus T$ if $\forall I \in M, \mmeval{\KB a} = \uu$
\end{itemize}

While every MKNF interpretation $(M, N)$ induces a unique partition $(T, P)$, in general, an MKNF interpretation that induces a given partition $(T, P)$ is not guaranteed to exist \cite{liuthreevalued2017}.
We say a partial partition can be \textit{extended} to an MKNF interpretation if there exists an MKNF interpretation that induces it.

\section{A Fixpoint Characterization}\label{section-main}
Before we give our characterization, we identify a subclass of partitions that is consistent with the ontology and where no immediate consequences can be derived from the ontology.
\begin{definition}
	We call a partition $(T, P)$ \textit{saturated} if $\OB{P}$ is consistent and $\OB{T} \not\models a$ for each $\KB a \in \KBA(\KBB) \setminus T$ and $\OB{P} \not\models a$ for each $\KB a \in \KBA(\KBB) \setminus P$.
\end{definition}

Intuitively, if a partition is saturated, then it is consistent with the ontology and $\OB{T}$ cannot derive additional \KBatoms{}.
For an arbitrary partition $(T, P)$, it is either easy to extend $(T, P)$ to a saturated partition or it is easy to conclude that no MKNF model induces $(T, P)$.
\begin{definition}
	Given a program $\P$, a \textit{head-cut} $R$ is a set $R \subseteq \P \times \KBA(\KBB)$ such that for each pair $(r, \KB h) \in R$ we have $\KB h \in \head(r)$ and there is at most one pair $(r, \KB h)$ in $R$ for any $r \in \P$.
\end{definition}
Because of the restriction on the number of times a rule may appear in a head-cut, head-cuts can function as normal logic programs where the head of the rule is the single selected atom.
For a head-cut $R$, we use $\head(R)$ (resp. $\Rule(R)$) to denote the set $\{ \KB h ~|~ (r, \KB h) \in R \}$ (resp. $\{ r ~|~ (r, \KB h) \in R \}$).
Now we build upon the definition of a supporting set and its accompanying $Q$ operator as defined by Killen and You \citeyear{killenfixpoint2021}.

\begin{definition}\label{HC}
	Given a saturated partition $(T, P)$ of a knowledge base  ${\KBB = (\OO, \P)}$, we define the set $H^{(T, P)}_{\KBB}$ to be a set of head-cuts such that for each $R \in H^{(T, P)}_{\KBB}$,
	we have
	(i) \begin{align*}
		\forall r \in \P:~ (r \in \Rule(R) \iff \left( \begin{array}{l}
			\Big( \bodyp(r) \subseteq P \land \KB(\bodyn(r)) \intersect T = \emptyset \Big) ~\land \vspace{3.5pt} \\
			\Big( \head(r) \intersect T = \emptyset \lor \big(\bodyp(r) \subseteq T \land \KB(\bodyn(r)) \intersect P = \emptyset\big) \Big)
		\end{array} \right.
	\end{align*}
	 \begin{align*}
		\textrm{and (ii)~~~} \forall (r, \KB h) \in R:~ (\KB h \in P \land \bigg( \KB h \in T \iff \bodyp(r) \subseteq T \land \KB(\bodyn(r)) \intersect P = \emptyset \bigg))
	\end{align*}

\end{definition}

Note that (i) and (ii) may conflict and it may not be possible to construct a nonempty set $H^{(T, P)}_{\KBB}$.
We soon show (\Lemma{empty-is-model}) how we rely on this property.
% {\bf The structure of the statement is comprised of two conditions. Write (i) and (ii) to make it clear. But (i) says r is is in it if head(r) is false and body(r) is undefined. Then (ii) says such a r cannot be in it. (i) and (ii) are conflicting. You want to define R' by (i) and restrict it by (ii) to get R. }
% {{If head is false and body is undefined then H is empty}}
Intuitively, every head-cut in $H^{(T, P)}_{\KBB}$ contains every rule whose body is not false and additionally excludes the rules which contain true atoms in their head and whose bodies evaluate as undefined.
If, w.r.t. $(T, P)$, a rule $r$'s body is true (resp. undefined), then $r$ must be in a pair $(r, \KB h) \in R$ where $\KB h$ is true (resp. undefined).
By this construction, if a rule is not satisfied by $(T, P)$, then no valid head-cuts can be formed to meet the criteria of $H^{(T, P)}_{\KBB}$. The following lemma formalizes this property.
\begin{lemma}\label{empty-is-model}
	% \begin{theoremEnd}{lemma}\label{empty-is-model}
	For a saturated partition $(T, P)$ of $\KBA(\KBB)$ where $\KBB = (\OO, \P)$, the set $H^{(T, P)}_{\KBB}$ is empty if and only if for every MKNF interpretation $(M, N)$ that induces $(T, P)$, $(M, N)$ does not satisfy $\pi(\P)$.
\end{lemma}
% \end{theoremEnd}
\begin{proof}
	% \begin{proofEnd}
	($\Rightarrow$)
	Assume there exists an MKNF interpretation $(M, N)$ that induces $(T, P)$ such that $(M, N)$ satisfies $\pi(\P)$.\footnote{While not needed for this proof, $(M, N)$ will always exist when $\OB{P}$ is consistent, which is required for $(T, P)$ to be saturated.}
	We can construct a head-cut $R$ that includes every rule with a positive body that evaluates as either true or undefined and select a head \KBatom{} from $T$ or $P$ appropriately.
	By the construction of $H^{(T, P)}_{\KBB}$ the negative body of each of these rules also evaluates as true w.r.t. $(T, P)$.
	We exclude rules where the body evaluates as undefined while there are true atoms in the head.
	This head-cut is in $H^{(T, P)}_{\KBB}$, thus the set is nonempty.

	($\Leftarrow$)
	Given a head-cut $R \in H^{(T, P)}_{\KBB}$, let $(M, N)$ be an MKNF interpretation that induces $(T, P)$;
	We show $(M, N)$ satisfies $\pi(\P)$.
	Every rule that is excluded from $R$ was either excluded because its body is undefined while it has true atoms in its head or it was excluded because its body evaluates as false w.r.t. $(M, N)$.
	By the construction of $H^{(T, P)}_{\KBB}$, for each rule $r \in \Rule(R)$, we have $$(I, \MN, \MN)(\bigvee \head(r)) \geq (I, \MN, \MN)(\bigwedge \bodyp(r) \land \bigwedge \bodyn(r))$$
	Because every rule in $\P$ is satisfied by $(M, N)$, we have $(M, N) $ satisfies $ \pi(\P)$.
	% \end{proofEnd}
\end{proof}

We demonstrate the set $H^{(T, P)}_{\KBB}$ for a simple knowledge base.
\begin{example}
	Let $\KBB = (\OO, \P)$ where $\OO = \emptyset$ and $\P$ is defined as follows:
	\begin{align*}
		1: \KB a,~ \KB b & \leftarrow \KB c         
		&2: \KB x,~ \KB y & \leftarrow \KB p,~ \Not q
	\end{align*}
	Let $(T_1, P_1) = (\KBA(\KBB), \KBA(\KBB) )$. The partition that assigns every \KBatom{} to be true. Note that there is no MKNF model that induces $(T_1, P_1)$.
	We use the numbers to the left of each rule to identify a rule in a pair in a head-cut.
	We have $H^{(T_1, P_1)}_{\KBB} = \{ \{ (1, \KB a) \}, \{ \{ (1, \KB b) \} \}$.
	Head-cuts that include rule $2$ are not present in $H^{(T_1, P_1)}_{\KBB}$ because the body of rule $2$ is false w.r.t. $(T_1, P_1)$.
	Let $(T_2, P_2) = (\KBA(\KBB) \setminus \big\{ \KB q\big\}, \KBA(\KBB) )$. Like $(T_1, P_1)$, $(T_2, P_2)$ assigns every \KBatom{} to be true, except $\KB q$, which is assigned undefined.
	However, $H^{(T_1, P_1)}_{\KBB} = H^{(T_2, P_2)}_{\KBB}$.
	This time rule $2$ is excluded because its body is undefined w.r.t. $(T_2, P_2)$ while there are true atoms in its head.
	Let $(T_3, P_3) = (\KBA(\KBB) \setminus \big\{ \KB q,~ \KB x,~ \KB y \big\}, \KBA(\KBB) \setminus \{ \KB y \} )$.
	$(T_3, P_3)$ assigns $\KB a$, $\KB b$, $\KB c$, and $\KB p$, to be true, $\KB y$ to be false, and $\KB q$ and $\KB x$ to be undefined.
	Now rule $2$ is included, i.e., $H^{(T_3, P_3)}_{\KBB} = \{ \{ (1, \KB a), (2, \KB x) \}, \{ \{ (1, \KB b), (2, \KB x) \} \}$.
	Finally, let $(T_4, P_4) = (T_3, P_3 \setminus \{ \KB x \})$.
	We cannot construct a head-cut $R \in H^{(T, P)}_{\KBB}$ because rule $2$ must be in $R$, however, there is no head atom to select from the head of $2$.
\end{example}

Before we show how the set $H^{(T, P)}_{\KBB}$ relates to MKNF models, we need an operator that justifies atoms within a head-cut.
Intuitively, this operator takes a single induced normal logic program from $H^{(T, P)}_{\KBB}$ and iteratively accumulates \KBatom{}s in $P$.
Justification for undefined atoms may come from a rule with an undefined body or from the ontology whereas justification for true atoms can only come from a rule with a true body or the ontology paired with other true atoms that have already been derived.

\begin{definition}
	We define the following operator for a saturated partition $(T, P)$, a head-cut $R$ and a set of \KBatoms{} $S$:
	\begin{align*}
		Q^R_{(T, P)} (S)  = \bigg\{ \KB h ~|~ \textrm{where either } \left\{\begin{array}{ll}
			(r, \KB h) \in R,~ \bodyp(r) \subseteq S,~\textrm{or}        \\
			\KB h \in T,~ \OB{S \intersect T} \models h,~\textrm{or} \\
			\KB h \in P \setminus T,~\OB{S} \models h
		\end{array} \right. \bigg\}
	\end{align*}
\end{definition}

% While $Q^{R}_{(T, P)}$ operates on \KBatoms{}, we omit the $\KB$ when describing its output.
We have $\OB{S} \models a$ or $\OB{S \intersect T} \models a$ for each \KBatom{} $\KB a \in S$, thus $Q^R_{(T, P)}$ is monotonic w.r.t. the $\subseteq$ relation and a least fixpoint exists \citename{tarskilatticetheoretical1955}.
Intuitively, this operator cannot use undefined atoms to justify the derivation of true atoms. Rules that can derive true atoms must have a justified body, and the ontology is only given other true \KBatoms{} when deriving true atoms.
In the following, we demonstrate this operator with head-cuts from the set $H^{(T, P)}_{\KBB}$.
\begin{example}
	Define $\P$ as follows:
	\begin{align*}
		1: \KB a,~ \KB b                \leftarrow.        &&
		2: \KB c,~ \KB d                \leftarrow \Not z. &&
		3: \KB z \leftarrow \Not z.&&
	\end{align*}
	and let $(T, P) = (\{ \KB a, \KB b \}, \KBA(\KBB))$. Given $(T, P)$, one can see that $\KB a$ and $\KB b$ are true and $\KB c$, $\KB d$, and $\KB z$ are undefined.
	We first consider the knowledge base $\KBB_1 = (\OO_1, \P)$ where $\OO_1 = c \implies (b \land d)$.
	The set $H^{(T, P)}_{\KBB_1}$ is comprised of several head-cuts. Let us restrict our attention to the head-cut $R \in H^{(T, P)}_{\KBB_1}$ where $R = \{ (1, \KB a), (2, \KB c), (3, \KB z) \}$.
	The positive bodies of rules $1$, $2$, and $3$ are all empty, therefore $Q^R_{(T, P)} (\emptyset) = \{ \KB a, \KB c, \KB z \}$.
	On the operator's second iteration, it reaches the fixpoint $Q^R_{(T, P)} (\{ \KB z, \KB a, \KB c, \KB d \}) = \{ \KB z, \KB a, \KB c, \KB d \}$.
	Even though the ontology entails $c \implies b$, we do not derive $\KB b$ because it is a true atom whereas $\KB c$ is undefined (In the operator $\KB b \in T$ and $\OB{\{ a, c, z \} \intersect \{ a, b \}} \not\models b$).
	This mirrors our construction of the set $H^{(T, P)}_{\KBB}$ where rules that have true atoms in their head while having an undefined body are removed from head-cuts.
	If we replace the ontology with $\OO_2 = a \implies (b \land d)$ s.t. $\KBB_2 = (\OO_2, \P)$, then we find that $Q^R_{(T, P)} (\{\KB  z, \KB a, \KB c \}) = \{ \KB z, \KB a, \KB c, \KB b, \KB d \}$. This time $\KB d$ is derived because it is an undefined atom whose derivation by the ontology depends on a true atoms.
\end{example}

We now establish how this operator characterizes three-valued MKNF models of a disjunctive knowledge base.

% \begin{theoremEnd}{theorem}\label{main}
\begin{theorem}\label{main}
	Let $\KBB$ be a disjunctive hybrid MKNF knowledge base and $(M, N)$ be a three-valued MKNF interpretation pair that induces $(T, P)$.
	% such that for each $I \in M$, $I \models \pi(\OO)$.
	$(M, N)$ then is a three-valued MKNF model of $\KBB$ if and only if $(T, P)$ is a saturated partition s.t. for each $R \in H^{(T, P)}_{\KBB}$, $\lfp{Q^R_{(T, P)}} = P$, and $H^{(T, P)}_{\KBB} \not= \emptyset$.
\end{theorem}
% \end{theoremEnd}

\begin{proof}
	% \begin{proofEnd}
	($\Rightarrow$)
	Assume $(M, N)$ is an MKNF model of $\KBB$ that induces $(T, P)$.
	The partition $(T, P)$ induced by $(M, N)$ is unique. We show that $(T, P)$ is saturated.
	(1) $\OB{P}$ is consistent, that is, for each $I \in M$, $I \models \pi(\OO)$, for each $I \in N$, $I \models \{ a ~|~ \KB a \in P \}$, and $\emptyset \subset N \subseteq M$, thus $\OB{P}$ is consistent.
	(2) For each $\KB a \in \KBA(\KBB) \setminus T$, $\OB{T} \not\models a$, that is, 
	for each $\KB a \in \KBA(\KBB) \setminus T$, there is an interpretation $I \in M$ such that $I \not\models a$ and $I \models \OB{T}$, thus $\OB{T} \not\models a$.
	(3) For each $\KB a \in \KBA(\KBB) \setminus P$, $\OB{P} \not\models a$, that is, for every atom $\KB a \in \KBA(\KBB) \setminus P$, we have $I \in N$ where $I \not\models a$. We have $I \models \OB{P}$, thus $\OB{P} \not\models a$.
	With (1), (2), and (3), we've shown that $(T, P)$ is saturated.
	We apply \Lemma{empty-is-model} to conclude that $H^{(T, P)}_{\KBB}$ is nonempty.

	Finally, we show by contrapositive that for each $R \in H^{(T, P)}_{\KBB}$, $\lfp{Q^R_{(T, P)}} = P$.
	Let $R \in H^{(T, P)}_{\KBB}$ be a head-cut such that $\lfp{Q^R_{(T, P)}} \not= P$.
	We show that $(M, N)$ is not an MKNF model.
	Because $\head(R) \subseteq P$, the $Q^R_{T, P}$ operator cannot compute atoms that are not in $P$;
	thus $\lfp{Q^R_{(T, P)}} \subseteq P$.
	Let $S$ be the set of rules $r \in \P \setminus \Rule(R)$ such that $\bodyp(r) \subseteq P$ and $\KB(\bodyn(r)) \intersect T = \emptyset$.
	Intuitively, $S$ includes the rules that were excluded from $R$ because their bodies are evaluated as undefined while they have true atoms in their heads.
	By the definition of $H^{(T, P)}_{\KBB}$, for each rule $r \in S$, we have $\head(r) \intersect T \not= \emptyset$.
	Extend $R$ with pairs that contain these rules by selecting $h$ arbitrarily from $\head(r) \intersect T$. 
	Let	$R' = R \union \{ (r, h \in \head(r) \intersect T) ~|~ r \in S  \}$.
	Let $(T', P') = (T \intersect \lfp{Q^R_{(T, P)}}, \lfp{Q^{R'}_{(T, P)}})$.
	We have $P \supset \lfp{Q^{R'}_{(T, P)}} \supseteq \lfp{Q^R_{(T, P)}}$, thus $T' \subseteq T$ and $P' \subseteq P$.
	We construct a pair
	$$(M', N') = (\{ I ~|~ \OB{T'} \models I  \}, \{ I ~|~ \OB{P'} \models I  \})$$
	$\OB{P'}$ is consistent, thus $N' \not= \emptyset$ and clearly, $M' \subseteq N'$, thus $(M', N')$ is an MKNF interpretation.
	$(M', N')$ induces $(T', P')$ and because $T' \subseteq T$ and $P' \subseteq P$ (and $(T', P') \not= (T, P)$), we have $M' \supseteq M$ and $N' \supseteq N$ (and $(M', N') \not= (M, N)$).
	We show that for each $I \in M'$, $${(I, \langle M', N' \rangle, \MN)(\pi(\KBB)) = \tt}$$
	Let $I \in M'$.
	We divide $\pi(\KBB)$ into (1) $\pi(\OO)$ and (2) $\pi(\P)$.
	(1) By the construction of $(M', N')$, we have $\forall J \in M'$, $J \models \pi(\OO)$.
	(2) We show $(I, \langle M', N' \rangle, \MN)(\pi(\P)) = \tt$, i.e., that for each $r \in \P$, $(I, \langle M', N' \rangle, \MN)(\pi(r)) = \tt$.
	We consider three separate cases for rules $r \in \P$.
	(i) $\bodyp(r) \not\subseteq P'$ or $\KB(\bodyn(r)) \intersect T \not= \emptyset$.
	This rule's body is false w.r.t. $(I, \langle M', N' \rangle, \MN)$, thus the rule is satisfied.
	(ii) $r \in \Rule(R)$, $\bodyp(r) \subseteq P'$ and $\KB(\bodyn(r)) \intersect T = \emptyset$.
	The negative body of $r$ only contains $\Not$ atoms so it is evaluated against $(M, N)$ which induces $(T, P)$.
	The $Q^{R}_{(T, P)}$ operator selects some atom from $\head(r)$ that will ensure this rule is satisfied.
	(iii) $r \in R' \setminus R$,  $\bodyp(r) \subseteq P'$, and $\KB(\bodyn(r)) \intersect T' = \emptyset$.
	We have either $\bodyp(r) \not\subseteq  T' \subseteq T$ or $\KB(\bodyn(r)) \intersect P$.
	Thus,  ${(I, \langle M', N' \rangle, \MN)(\bodyp(r) \land \KB(\bodyn(r))) = \uu}$.
	The $Q^{R'}_{(T, P)}$ operator will compute a head-atom $h$ from $\head(r)$.
	We have $h \in P'$ and $h \not\in T'$, thus  $$(I, \langle M', N' \rangle, \MN)(\KB h) = \uu$$
	This is sufficient to show that $\pi(r)$ is satisfied.
	The cases above are sufficient to show that $(M, N)$ is not an MKNF model of $\KBB$.

	($\Leftarrow$)
	By contrapositive.
	Assume that $(M, N)$ is not an MKNF model of $\KBB$, we show that either there exists $R \in H^{(T, P)}_{\KBB}$ such that $\lfp{Q^R_{(T, P)}} \not= P$, $(T, P)$ is not a saturated partition, or $H^{(T, P)}_{\KBB} = \emptyset$.
	Assume for the sake of contradiction, that $\forall R \in H^{(T, P)}_{\KBB}, \lfp{Q^R_{(T, P)}} = P$, $(T, P)$ is saturated and ${H^{(T, P)}_{\KBB} \not= \emptyset}$.
	We show that there exists an $R \in H^{(T, P)}_{\KBB}$ such that $\lfp{Q^R_{T, P}} \subset P$, a contradiction.
	Because $(M, N)$ is not an MKNF model of $\KBA(\KBB)$, either $(M, N) $ does not satisfy $ \pi(\P)$, $(M, N) $ does not satisfy $ \pi(\OO)$, or there exists an MKNF interpretation $(M', N')$ such that $M' \supseteq M$ and $N' \supseteq N$ (and $(M', N') \not= (M, N)$).
	Apply \Lemma{empty-is-model} with $H^{(T, P)}_{\KBB} \not= \emptyset$ to rule out the possibility that $(M, N) $ does not satisfy $ \pi(\P)$.
	% TODO maybe drop that old condition?
	That $(M, N) $ does not satisfy $ \pi(\OO)$ is also out of the question because $(T, P)$ is stagnant therefore $\pi(\OO) \models I$ for each $I \in M$.
	With the above, we assume that such an MKNF interpretation $(M', N')$ exists.
	We have for each $r \in \P$ and each $I \in M'$, $(I, \langle M', N', \rangle, \MN)(\pi(r)) = \tt$.
	First, we focus on the MKNF interpretation $(M', N)$ which has the following property:
	$$\forall I \in M', (I, \langle M', N \rangle, \MN)(\pi(\KBB)) = \tt$$
	% \begin{itemize}
	% 	\item 
	% 	\item $\forall\, \KB a \in T,  I \in M', (I, \langle M', N \rangle, \MN)(\KB a) \not= \ff$
	% \end{itemize}
	
	We consider the case that $M' \not= M$, and derive a contradiction.
	We can construct a head-cut $R \in H^{(T, P)}_{\KBB}$ such that $\lfp{Q^R_{(T, P)}} = P \setminus S$ where $S$ contains all the atoms that evaluate as true under $(M, N)$ but undefined under $(M', N)$.
	This is done by avoiding picking atoms in $S$ for $\head(R)$ whenever possible.
	There is a case where a rule $r \in \P$ has had all of its head atoms changed from true to undefined and therefore we must include a pair $(r, \KB h)$ where $h \in S$.
	However, the positive body of this rule will not be contained by $\lfp{Q^R_{(T, P)}}$, therefore $h$ will not be computed.
	We have $\lfp{Q^R_{(T, P)}} \not= P$, a contradiction.
	We assume $M' = M$ and repeat a similar process for all the atoms that are true under $(M, N)$ but false under $(M, N')$ and construct a head-cut $R$ such that $\lfp{Q^R_{(T, P)}} \not= P$, a contradiction.
	% \end{proofEnd}
\end{proof}

Given a saturated partition $(T, P)$, this theorem states that we can determine whether there is a three-valued MKNF model that induces it by enumerating all head-cuts in $R \in H^{(T, P)}_{\KBB}$ and checking that $\lfp{Q^R_{(T, P)}} = P$.
We do not need to check \KBatoms{} in $T$ because the construction has no way to derive \KBatoms{} from $T$ as undefined.
Because three-valued semantics reduces to two-valued semantics for MKNF interpretations of the form $(M, M)$ \citename{knorrlocal2011}, this operator can also check two-valued MKNF models.
Below, we demonstrate \Theorem{main} in action.
Let us first consider the special case where rules in a hybrid MKNF knowledge base are normal.
\begin{example}
	Consider $\KBB = (\OO, \P)$ where $\OO = a \land b \supset c$ and $\P$ is defined as follows:
	\begin{align*}
		1: \KB a  \leftarrow.        &&
		2: \KB b  \leftarrow \Not b. &&
		3: \KB c  \leftarrow \KB c.  &&
	\end{align*}
	Let $(T, P) = (\{ \KB a  \}, \{ \KB a, \KB b, \KB c \})$.
	$(T, P)$ is a saturated partition.
	The $Q_{(T,P)}^{R}$ operator can be used for model-checking.
	There exists an MKNF model $(M, N)$ that induces $(T, P)$. We show that \Theorem{main} agrees.
	The set $H^{(T, P)}_{\KBB}$ contains a single head-cut $R = \{ (1, \KB a), (2, \KB b), (3, \KB c) \}$.
	We have $\lfp{Q^R_{(T, P)}} = \{ \KB a, \KB b, \KB c \} = P$.

	\Theorem{main} can also show that there does not exist an MKNF model that induces a partition $(T, P)$.
	Let $(T', P') = (\{ \KB a, \KB c  \}, \{ \KB a, \KB b, \KB c \})$.
	$(T', P')$ cannot be extended to an MKNF model and $H^{(T', P')}_{\KBB} = H^{(T, P)}_{\KBB}$.
	We have ${\lfp{Q^R_{(T', P')}} = \{ \KB a, \KB b \} \not= P}$.
	The \KBatom{} $\KB c$ is not computed because $\OB{\{a,b\}\intersect \{a, c\}} \not\models c$.
\end{example}

Now let us consider a disjunctive knowledge base where the ontology is empty, which shows that the operator can be applied to disjunctive logic programs.
\begin{example}
	Let $\KBB = (\OO, \P)$ where $\OO = \emptyset$ and $\P$ is defined as follows:
	\begin{align*}
		1: \KB a, \KB b,                & \leftarrow \Not d 
		&2: {\color{white}\KB a,~} \KB a & \leftarrow \KB b  
		&3: {\color{white}\KB b,~} \KB b & \leftarrow \KB a  \\
		4: \KB d,~ \KB c                & \leftarrow        
		&5: {\color{white}\KB d,~} \KB d & \leftarrow \Not d
		% &5: {\color{white}\KB d,~} \KB d & \leftarrow \Not d
	\end{align*}

	Consider the partition $(T, P) = (\{ \KB c \}, \{ \KB d, \KB a, \KB b, \KB c \})$.
	The set $H^{(T, P)}_{\KBB}$ contains the following head-cuts:
	\begin{align*}
		R_1 & = \big\{ (1, \KB a),~ (2, \KB a),~ (3, \KB b),~ (4, \KB c),~  (5, \KB d)  \big\} \\
		R_2 & = \big\{ (1, \KB b),~ (2, \KB a),~ (3, \KB b),~ (4, \KB c),~ (5, \KB d)  \big\}
	\end{align*}

	The pair $(4, \KB d)$ does not occur in any head-cut in $H^{(T, P)}_{\KBB}$ because $\KB d$ is undefined while the body of rule $4$ is true.
	When we apply the operator to each head-cut in $H^{(T, P)}_{\KBB}$ we get  $\lfp{Q^{R_i}_{(T, P)}} = P$ for each $R_i$, thus we confirm that there is an MKNF model that induces $(T, P)$.
\end{example}

Finally, we provide an example with a disjunctive program and an ontology.
\begin{example}
	Let $\KBB = (\OO, \P)$.
	where $\OO = (a \lor b) \land (x \lor y) \implies a \land b \land x \land y$
	and $\P$ is defined as follows:
	\begin{align*}
		1: \KB a,~ \KB b & \leftarrow       
		&2: \KB x,~ \KB y & \leftarrow \Not x
	\end{align*}

	For both disjunctive rules $1$ and $2$, there is potential for them to contain multiple non-false \KBatoms{} in their heads.
	Let $(T_1, P_1) = (\KBA(\KBB) \setminus \{ \KB x, \KB y \}, \KBA(\KBB)) = (\{ \KB a, \KB b \}, \{ \KB a, \KB b, \KB x, \KB y \})$, the partition that assigns $\KB a$ and $\KB b$ to be true and $\KB x$ and $\KB y$ to be undefined.
	There is no three-valued MKNF model that induces $(T_1, P_1)$.
	The set $H^{(T_1, P_1)}_{\KBB}$ contains the following:
	\begin{align*}
		R_1 & = \big\{ (1, \KB a), (2, \KB x)  \big\} & R_2 & = \big\{ (1, \KB b), (2, \KB x)  \big\} \\
		R_3 & = \big\{ (1, \KB a), (2, \KB y)  \big\} & R_4 & = \big\{ (1, \KB b), (2, \KB y)  \big\}
	\end{align*}
	When we apply the operator to each head-cut in $H^{(T_1, P_1)}_{\KBB}$ we get the following
	least fixpoints:
	\begin{align*}
		\lfp{Q^{R_1}_{(T_1, P_1)}} & = \big\{ \KB a, \KB x, \KB y  \big\} & \lfp{Q^{R_2}_{(T_1, P_1)}} & = \big\{ \KB b, \KB x, \KB y  \big\} \\
		\lfp{Q^{R_3}_{(T_1, P_1)}} & = \big\{ \KB a, \KB x, \KB y  \big\} & \lfp{Q^{R_4}_{(T_1, P_1)}} & = \big\{ \KB b, \KB x, \KB y  \big\}
	\end{align*}

	None of which is equal to $P$.
	If we remove $b$ from the true atoms in $(T_1, P_1)$ and make it false to form $(T_2, P_2)$, i.e., $(T_2, P_2) = (T \setminus \{ \KB b\}, P \setminus \{ \KB b\}) = (\{ \KB a \}, \{ \KB a, \KB x, \KB y \})$.
	Then we have $\OB{P} \models b$. This shows that $(T_2, P_2)$ is not saturated and therefore we cannot apply \Theorem{main} to it.
	If $(T_2, P_2)$ were induced by an MKNF model $(M, N)$, then there would be an $I \in N \intersect M$ s.t. $I \not\models b$.
	It follows that $I \not\models \OO$, is a contradiction.
	Checking whether a partition is saturated is an important step because if applied to $(T_2, P_2)$, our operator would compute $P_2$.
	Instead, lets make $b$ undefined, i.e., lets fix $(T_3, P_3)$ to be $(\{ \KB a \}, \KBA(\KBB))$.
	Then we have the following head-cuts in $H^{(T_3, P_3)}_{\KBB}$:
	\begin{align*}
		R_1 & = \big\{ (1, \KB a), (2, \KB x)  \big\} & R_2 & = \big\{ (1, \KB a), (2, \KB y)  \big\}
	\end{align*}

	The least fixpoints for each of these head-cuts follow:
	\begin{align*}
		\lfp{Q^{R_1}_{(T_3, P_3)}} & = \big\{ \KB a, \KB b, \KB x, \KB y  \big\} & \lfp{Q^{R_2}_{(T_3, P_3)}} & = \big\{ \KB a, \KB b, \KB x, \KB y  \big\}
	\end{align*}

	Each fixpoint is equal to $P_3$ and there is an MKNF model that induces $(T_3, P_3)$:
	$$(M, N) = (\{ I ~|~ \OB{T_3} \models I \}, \{ I ~|~ \OB{P_3} \models I \})$$
	Each rule is satisfied by $(M, N)$ and the first-order interpretation that assigns both $\KB x$ and $\KB y$ to be false is in $M$, thus $(M, N)$ does not satisfy $\KB (x \lor y)$.
	This shows that $(M, N)$ satisfies $\pi(\OO)$.
\end{example}

Using \Theorem{main} we can derive some interesting implications on the relationship between normal and disjunctive knowledge bases.
We say that a head-cut $\P'$ of a program $\P$ is \textit{total} if it contains every rule in $\P$, that is, $\Rule(R) = \P$.
Given a disjunctive knowledge base $(\OO, \P)$, we can construct a normal knowledge base $(\OO, \P')$ from a total head-cut $\P'$ of  $\P$.
We call such a knowledge base an \textit{induced normal knowledge base} of $\KBB$.
We show how the MKNF-consistency of an induced normal knowledge base of $\KBB$ relates to the MKNF-consistency of $\KBB$.

\begin{corollary}\label{model-amongst-normals}
	Given a disjunctive hybrid MKNF knowledge base $\KBB = (\OO, \P)$, let $(M, N)$ be a three-valued MKNF model of $\KBB$ that induces the partition $(T, P)$.
	Let $\KBB' = (\OO, \P')$ be an induced normal knowledge base of $\KBB$.
	$(M, N)$ is a three-valued MKNF model of $\KBB'$ if and only if there is a head-cut $R \in H^{(T, P)}_{\KBB}$ such that $R \subseteq \P'$.
\end{corollary}
\begin{proof}
	($\Rightarrow$)
	Assume $(M, N)$ is an MKNF model of $\KBB'$. Then $H^{(T, P)}_{\KBB'} = \{ R \}$. We have $R \subseteq \P'$ and $R \in H^{(T, P)}_{\KBB}$.
	($\Leftarrow$)
	Assume $R \in H^{(T, P)}_{\KBB}$ s.t. $R \subseteq \P'$. We have $\lfp{Q^R_{(T, P)}} = P$ and $\{ R \} = H^{(T, P)}_{\KBB'}$.
	By \Theorem{main}, there exists an MKNF model $(M', N')$ of $\KBB'$ that induces $(T, P)$.
	Because the ontologies are the same in $\KBB$ and $\KBB'$, $(M, N) = (M', N')$.
\end{proof}

We can generalize the previous corollary slightly to make conclusions about MKNF-inconsistent induced normal knowledge bases.
\begin{corollary}
	Let $\KBB = (\OO, \P)$ be a disjunctive hybrid MKNF knowledge base and let $\KBB' = (\OO, \P')$ be an induced normal knowledge base of $\KBB'$.
	Let $(M, N)$ be an MKNF interpretation that induces a saturated partition $(T, P)$.
	If $\KBB'$ is MKNF-inconsistent, and there exists a head-cut $R \in H^{(T, P)}_{\KBB}$ s.t. $\Rule(R) \subseteq \P'$, then $(M, N)$ is not an MKNF model of $\KBB$.
\end{corollary}

% \begin{proof}
%     By contrapositive: Assume $(M, N)$ is an MKNF model of $(\OO, \P)$.
%   By \Theorem{main}, $H^{(T, P)}_{\KBB}$ is non-empty. $\P'$ contains every rule, thus there is some head-cut $R \in H^{(T, P)}_{\KBB}$ s.t. $R \subseteq \P'$.
%   We have $H^{(T, P)}_{(\OO, \P')} = \{ R \}$ and $\lfp{Q^R_{(T, P)}} = P$ therefore, by \Theorem{main}, $(M, N)$ is an MKNF model of $(\OO, \P')$.
% \end{proof}
These properties add to the theory of disjunctive hybrid MKNF knowledge bases.
We can eliminate some MKNF models by guessing MKNF-inconsistent induced normal knowledge bases of a disjunctive knowledge base.

\section{Relationship with Approximators in AFT}\label{section-aft}
By applying the result obtained in the previous section, we demonstrate a link between three-valued MKNF models of a disjunctive hybrid MKNF knowledge base and the stable fixpoints of AFT approximators  for induced normal knowledge bases.
%AFT is an algebraic framework for the study of fixpoints of operators on bilattices.

Given a complete lattice $\langle L, \leq\rangle $,  AFT is built on the induced product bilattice
$\langle L^2, \leq_p\rangle$, where $\leq_p$ is called the {\em precision order} and defined as
for all $x, y, x', y' \in L$,
$(x, y) \leq_p (x', y')$ if $x \leq x'$ and $y' \leq y$. A pair $(x,y) \in L^2$ is {\em consistent} if $x \leq y$ and {\em inconsistent} otherwise. Since
the $\leq_p$ ordering is a complete lattice ordering on $L^2$, $\leq_p$-monotone operators on $L^2$
contain fixpoints and a least fixpoint.
%Below, we often write a lattice $\langle L, \leq\rangle $ by $L$ and its induced product bilattice by  $L^2$.
The original AFT is restricted to consistent and symmetric approximators \cite{denecker2000approximations, DeneckerMT04}, and Liu and You generalize it to all $\leq_p$-monotone operators on $L^2$. Working with an ontology, which can result in inconsistencies, warrants supporting inconsistent pairs.
%support the operators that map a consistent state to an inconsistent one, and even allow inconsistent stable fixpoints. This is motivated by the possible role that AFT may play in building constraint propagators for solvers of an underlying logic (e.g., \cite{Ji17}), where inconsistency not only guides the search via backtracking but also provides valuable information to prune the search space. One can also argue that inconsistent stable fixpoints may provide useful information for debugging purposes.
%As alluded earlier, a critical property of a symmetric approximator  is that it maps an exact pair to an exact pair.  However, a $\leq_p$-monotone operator on $L^2$ may not possess this property. 
\begin{definition}[Liu and You 2021]
	An operator $A: {L}^2\rightarrow {L}^2$ is an approximator %for an operator $O$ on $L$,
	if $A$ is $\leq_p$-monotone on $L^2$ and for all $x \in {L}$, and whenever ${A}(x,x)$ is consistent, $A$ maps $(x,x)$ to an exact pair.
	%Let $O$ be an operator on $L$.  We say that  $A: L^2 \rightarrow L^2$ is an {\em approximator for $O$} if $A$ is an approximator and for all $x \in {L}$, if ${A}(x,x)$ is consistent then  ${A}(x,x) = ({O}(x), {O}(x))$.
\end{definition}

For the study of semantics, we focus on the {\em stable revision operator}, which we define below: Given any pair $(u,v) \in L^2$ and an approximator $A$, we define
\begin{eqnarray} St_A(u,v) =
	(\lfp({A}( \cdot , v)_1), \lfp(A(u, \cdot)_2))
	\label{stable-revision}
\end{eqnarray}
where  ${A}( \cdot , v)_1$ denotes the operator $L \rightarrow L: z \mapsto A(z, v)_1$ and $A(u,\cdot)_2$ denotes the operator $L \rightarrow L: z \mapsto A(u, z)_2$. That is, both
${A}( \cdot , v)_1$ and $A(u,\cdot)_2$ are projection operators defined on $L$.  It can be shown that since $A$ is $\leq_p$-monotone on $L^2$, both projection operators ${A}( \cdot , v)_1$ and $A(u,\cdot)_2$ are
$\leq$-monotone on $L$ for any pair in $L^2$ and thus  the least fixpoint exists for each.  The stable revision operator is thus well-defined. It can be shown further that the stable revision operator is $\leq_p$-monotone. The fixpoints of the stable revision operator $St_{A}$ are called {\em stable fixpoints} of $A$.\footnote{For normal logic programs, for example, it is known  \cite{denecker2000approximations} that Fitting's  $\Psi_{\P}$ operator \cite{Fitting02} is in fact an approximator whose least fixpoint corresponds to the well-founded model and stable fixpoints correspond to three-valued stable models.}%\footnote{In the extention by Liu and You %\citeyear{liuyou2021} even inconsistent stable fixpoints are now possible. For example,  with ${\KBB} = (\{\neg a\}, \{a \leftarrow {\boldnot} b.\})$,  $(\{{\boldK} \, a, {\boldK} \, b\}, \emptyset)$ is  a stable fixpoint of $\Phi_{{\cal K}_1}$.}

For normal hybrid MKNF knowledge bases, Liu and You define the following approximator.
\begin{definition}[Liu and You 2021]
	\label{Phi}
	Let $\KBB = (\cal O, \P)$ be a normal hybrid MKNF knowledge base.
	We define an operator $\Phi_{\KBB}$ on $(2^{\KBA(\KBB)})^2$ as follows:
	%given  a pair $(T,P) \in (2^{\KBA(\KB)})^2$,  we define
	$\Phi_{\KBB}(T,P) = (\Phi_{\KBB} (T,P)_1 , \Phi_{\KBB}(T,P)_2)$,
	where
	$$
		\begin{array}{ll}
			\Phi_{\KBB} (T,P)_1 =  \{{\boldK} \, a\in {\KBA}({\KBB}) \mid {\OBB}_{{\cal O}, T}\models a\} \, \cup \,
			\\~~~~~~~~~~~~~~~~~~~~~~~~~~~~~~~\{ {\bfK}\, a \mid  r \in {\P}:  \, head(r) = \{{\bfK}\, a \}, \, body^+(r) \subseteq T, \, {\boldK}(body^-(r)) \cap P = \emptyset\}
			\\
			\Phi_{\KBB} (T,P)_2 =  \{{\boldK} \,a\in {\KBA}({\KBB}) \mid {\OBB}_{{\cal O}, P}\models a\} \,\,  \cup \,
			\\~~~~~~~~~~~~~~~~~~~~~~~~~~~~~~~\{ {\bfK}\, a \mid  r \in {\P}:  \, head(r) = \{{\bfK}\, a \}, \,
			{\OBB}_{{\cal O}, T} \not \models \neg a,  \, body^+(r) \subseteq P, \, \\ ~~~~~~~~~~~~~~~~~~~~~~~~~~~~~~~~~~~~~~~~~~~~~{\boldK}(body^-(r)) \cap T = \emptyset\}
		\end{array}
	$$
\end{definition}

Intuitively, given a partition $(T,P)$, the operator $\Phi_{\KBB}(\cdot,P)_1$, with $P$ fixed, computes the set of true modal $\bfK$-atoms w.r.t.~$(T,P)$
and operator $\Phi_{\KBB}(T,\cdot)_2$, with $T$ fixed, computes the set of modal $\bfK$-atoms that are possibly true w.r.t.~$(T,P)$. The condition ${\OBB}_{{\cal O}, T} \not \models \neg a$ attempts to avoid the generation of a contradiction.
Liu and You \citeyear{liuyou2021} show that
$\Phi_{\KBB}$ is an approximator on the bilattice $(2^{\KBA(\KBB)})^2$ and preserves all consistent stable fixpoints when restricted to consistent pairs. Note that operator ${\Phi}_{\KBB}$ is not symmetric and it can map a consistent pair to an inconsistent one.

\begin{example}[Liu and You 2021]
	Consider a normal hybrid MKNF knowledge base $\KBB =({\cal O},{\P})$, where ${\cal O} = a \wedge (b \supset  c) \wedge \neg f$ and
	$\P$ is
	% $$
	% \begin{array}{ll} 
	% {\boldK} b \leftarrow {\boldK} a.  ~~~
	% {\boldK} d \leftarrow {\boldK}c, {\boldnot} e. ~~~{\boldK} e \leftarrow {\boldnot}d. ~~~{\boldK}f \leftarrow {\boldnot} b.
	% \end{array}
	% $$
	\begin{align*}
		\KB b \leftarrow \KB a.          &&
		\KB d \leftarrow \KB c,~ \Not e. &&
		\KB e \leftarrow \Not d.         &&
		\KB f \leftarrow \Not b.			&&
	\end{align*}
	Reasoning with $\KBB$ can be seen as follows:
	since $\boldK \pi({\cal O})$ implies ${\boldK}\, a$, by
	the first rule we derive ${\boldK} b$, then due to $b \supset c$ in $\cal O$ we derive
	${\boldK} \, c$. Thus its occurrence in the body of the second rule is true and can be ignored. For the ${\bf K}$-atoms ${\boldK}\,  d$ and ${\boldK}\, e$ appearing in the two rules in the middle,  without preferring one over the other, both can be undefined.
	Because ${\boldnot} b$ is false (due to $\neg f$ in $\cal O$), the last rule is also satisfied.
	Now
	consider an MKNF interpretation $(M,N)=(\{I\,|\, I \models \pi({\cal O}) \wedge b \},
	\{I\,|\,I\models \pi({\cal O}) \wedge b \wedge d \wedge e\})$, which induces the partition $(T,P) = (\{\KB a, \KB b,
	\KB c\}, \{\KB a, \KB b, \KB c, \KB d, \KB e\})$.
One can verify that $(T,P)$ is a stable fixpoint of $\Phi_{\KBB}$, i.e.,
	$St_{\Phi_{\KBB}} (T,P) =  (\lfp({\Phi_{\KBB}}( \cdot , P)_1), \lfp(\Phi_{\KBB} (T, \cdot)_2)).$
	%For instance, we have that, for all $I \in M$, $(I, \langle M,N\rangle, \langle M,N \rangle)({\boldK} a)=\bf t$ and $(I, \langle M,N\rangle, \langle M,N \rangle) ({\boldK} d)=\bf u$.  The interpretation pair $(M,N)$ is a three-valued MKNF model of $\KBB$. %; in fact, it is the well-founded MKNF model of $\KBB$.
	%\qed
\end{example}

\begin{theorem}[Liu and You 2021] 
	\label{key-result}
	Let $\KBB = (\cal O,\P)$ be a normal hybrid MKNF knowledge base and $(T,P)$ be a partition.
	Also let $( M , N ) = (\{ I\, |\, I \models {\OBB}_{{\cal O} , T}\},\{I\,|\,I \models {\OBB}_{{\cal O},P}\})$.
	Then, $(M,N)$ is a three-valued MKNF model of $\KBB$ iff $(T,P)$ is a consistent stable fixpoint of $\Phi_{\KBB}$  and
	${\OBB}_{{\cal O}, \lfp(\Phi_{\KBB}(\cdot,T)_1)}$ is satisfiable.
\end{theorem}

In general, a stable fixpoint of operator $\Phi_{\KBB}$ may not correspond to an MKNF model.  It is guaranteed under the condition
that ${\OBB}_{{\cal O}, \lfp(\Phi_{\KBB}(\cdot,T)_1)}$ is satisfiable, which intuitively says that even if we allow all non-true
${\boldK}$-atoms to be false, we still cannot derive a contradiction.

The following theorem follows from Theorems \ref{main} and \ref{key-result} and \Corollary{model-amongst-normals}.

\begin{theorem}
	Let $\KBB = ({\cal O}, \P)$ be a disjunctive hybrid MKNF knowledge base, $(M,N)$ an MKNF interpretation of $\KBB$, and
	$(T,P)$ be induced from $(M,N)$.
	We have
	for each normal knowledge base $\KBB'$ induced by $\KBB$, $(T,P)$ is a stable fixpoint
	of $\Phi_{\KBB'}$ and ${\OBB}_{{\cal O}, \lfp(\Phi_{\KBB'}(\cdot,T)_1)}$ is satisfiable iff
	(i) $(T, P)$ is saturated,
	(ii) for each $R \in H^{(T, P)}_{\KBB}$, $\lfp{Q^R_{(T, P)}} = P$,
	and (iii) $H^{(T, P)}_{\KBB} \not= \emptyset$.
\end{theorem}
\begin{proof}
	Because $(T, P)$ is induced by an MKNF interpretation, it is consistent.
	Assuming the left condition, apply \Corollary{model-amongst-normals} and \Theorem{key-result} to conclude that $(M, N)$ is an MKNF model of $\KBB$, apply \Theorem{main} to obtain the right side.
	Assuming the right condition, apply \Theorem{main}. $(M, N)$ is an MKNF model of $\KBB$. Apply \Theorem{key-result} to obtain show left. 
\end{proof}

This result applies to disjunctive logic programs since they are a special case of disjunctive hybrid MKNF knowledge base with an empty ontology. The result also applies to the stable model semantics for disjunctive logic programs since two-valued stable models are a special case of three-valued stable models.

\section{Discussion}\label{section-discussion}
We've presented an operator that can be applied to the head-cuts of a disjunctive knowledge base to characterize its MKNF models.
By computing the fixpoint of an operator for every head-cut, we can confirm that a partition can be extended to an MKNF model.
Model-checking normally requires an NP-oracle when the ontology's entailment relation can be computed in polynomial time \citename{motikreconciling2010}. This suggests the complexity of using our operator for model-checking. The size of the set $H^{(T, P)}_{\KBB}$ is directly responsible for this complexity.
The techniques applied by Killen and You \citeyear{killenfixpoint2021} to reduce the size of this set can be applied to our operator for three-valued MKNF models with very little modification.
This would allow model-checking to be performed in polynomial time for a class of partial partitions analogous to the class Killen and You identified for the two-valued case, however, techniques to efficiently recognize members of this class of partitions still need to be developed.
This class is directly related to the class of head-cycle free disjunctive logic programs, in which models can be checked in polynomial time \citename{headcycleisNP}.

Finally, we have shown the close relationship between the MKNF models of a disjunctive knowledge base and the MKNF models of its induced normal logic knowledge bases.
We can apply existing AFT theory on induced normal knowledge bases to draw conclusions about MKNF models of the disjunctive knowledge base. 
Our construction can also be applied to partial stable semantics \citename{przymusinskistable1991} because three-valued MKNF models of disjunctive hybrid MKNF knowledge bases without ontologies coincide with the partial stable models of disjunctive logic programs.

\nocite{*}
\bibliographystyle{eptcs}
\bibliography{refs}
\newpage

\end{document}